\documentclass[nohyperref]{article}

\usepackage{microtype}
\usepackage{graphicx}
\usepackage{subfigure}
\usepackage{booktabs} 

\usepackage{hyperref}

\usepackage[accepted]{icml2022}

\usepackage{amsmath}
\usepackage{amssymb}
\usepackage{mathtools}
\usepackage{amsthm}
\usepackage{bm}
\usepackage[capitalize,noabbrev]{cleveref}

\theoremstyle{plain}
\newtheorem{theorem}{Theorem}[section]
\newtheorem{proposition}[theorem]{Proposition}
\newtheorem{lemma}[theorem]{Lemma}
\newtheorem{corollary}[theorem]{Corollary}
\theoremstyle{definition}
\newtheorem{definition}[theorem]{Definition}
\newtheorem{assumption}[theorem]{Assumption}
\theoremstyle{remark}

\newtheorem{example}[theorem]{Example}

\newcommand{\norm}[1]{\left\lVert#1\right\rVert}
\newcommand{\brac}[1]{\left(#1\right)}
\newcommand{\R}{\mathbb{R}}
\newcommand{\C}{\mathbb{C}}
\newcommand{\D}{\mathbf{D}}

\usepackage[textsize=tiny]{todonotes}

\icmltitlerunning{On Numerical Integration in Neural ODE}

\begin{document}

\twocolumn[
\icmltitle{On Numerical Integration in Neural Ordinary Differential Equations}




\begin{icmlauthorlist}
\icmlauthor{Aiqing Zhu}{amss,ucas}
\icmlauthor{Pengzhan Jin}{pu}
\icmlauthor{Beibei Zhu}{ustb}
\icmlauthor{Yifa Tang}{amss,ucas}
\end{icmlauthorlist}
\icmlaffiliation{amss}{LSEC, ICMSEC, Academy of Mathematics and Systems Science, Chinese Academy of Sciences, Beijing 100190, China}
\icmlaffiliation{ucas}{School of Mathematical Sciences, University of Chinese Academy of Sciences, Beijing 100049, China}
\icmlaffiliation{pu}{School of Mathematical Sciences, Peking University, Beijing 100871, China}
\icmlaffiliation{ustb}{School of Mathematics and Physics, University of Science and Technology Beijing, Beijing 100083, China}

\icmlcorrespondingauthor{Yifa Tang}{tyf@lsec.cc.ac.cn}

\icmlkeywords{Machine Learning, ICML}

\vskip 0.3in
]



\printAffiliationsAndNotice{}  

\begin{abstract}
The combination of ordinary differential equations and neural networks, i.e., neural ordinary differential equations (Neural ODE), has been widely studied from various angles. However, deciphering the numerical integration in Neural ODE is still an open challenge, as many researches demonstrated that numerical integration significantly affects the performance of the model. In this paper, we propose the inverse modified differential equations (IMDE) to clarify the influence of numerical integration on training Neural ODE models. IMDE is determined by the learning task and the employed ODE solver. It is shown that training a Neural ODE model actually returns a close approximation of the IMDE, rather than the true ODE. With the help of IMDE, we deduce that (i) the discrepancy between the learned model and the true ODE is bounded by the sum of discretization error and learning loss; (ii) Neural ODE using non-symplectic numerical integration fail to learn conservation laws theoretically. Several experiments are performed to numerically verify our theoretical analysis.
\end{abstract}

\section{Introduction}
Recently, Neural Ordinary Differential Equations (Neural ODE) \cite{chen2018neural} were proposed as a continuous model by embedding neural networks into continuous dynamical systems, and became an important option of model architecture. They offered dynamical systems perspectives on deep learning researches and thus have attracted increasing attention. For example, Yan et al. \yrcite{yan2020on} proposed TisODE to further enhance robustness according to the non-intersecting characteristics of ODE. Botev et al. \yrcite{botev2021priors} experimentally found that Neural ODE are the most effective approach to model latent dynamics from images due to continuous and time-reversible priors.

Despite ODE dynamics benefit models, we have to apply a numerical integration in Neural ODE, which prevents the model from being consistent with the design. Gusak et al. \yrcite{gusak2020towards} and Zhuang et al. \yrcite{zhuang2020adaptive} observed that changing the numerical solver yields performance degradation. Ott et al. \yrcite{ott2021resnet} and Queiruga et al. \yrcite{queiruga2020continuous} studied how the numerical integration influences the Neural ODE model and they proposed to adjust the step size and use high order solver to train Neural ODE, respectively. High-accuracy integration is able to loosen the impact of discretization error, however, quantifying such influences stills remains open. In this work, we also focus on the numerical integration in Neural ODE but we aim to decipher it theoretically and clarify the influence rigorously.

The main ingredient of this work is the formal analysis \cite{feng1991formal}. Historically, modified differential equation is an important tool for understanding the numerical behavior of solving ODE \cite{eirola1993aspects,feng1991formal,feng1993formal,sanz1992symplectic,yoshida1993recent}. The methodology is to interpret the numerical solution of the original system as the exact solution of a perturbed equation. Herein, this idea is tailored to Neural ODE. We first search for a perturbed differential equation such that its numerical solution matches the exact solution of the true system and then show that training Neural ODE returns a close approximation of this perturbed equation. The perturbed equation is named as inverse modified differential equation (IMDE) in this paper. Several experiments are performed to verify our theoretical analysis. Although the closed-form expression of the IMDE is inaccessible as it is calculated via the true system, we can still clarify the influence of numerical integration on Neural ODE with the help of IMDE. In summary, we list several statements derived via the proposed IMDE that will be documented in detail later:
\begin{itemize}
    \item The trained Neural ODE model is a close approximation of the IMDE, i.e., the difference between the learned Neural ODE model and the truncation of the IMDE is bounded by the sum of the learning loss and a discrepancy which can be made sub exponentially small.
    \item The difference between the learned Neural ODE model and the true hidden system is bounded by the sum of the discretization error $Ch^p$ and the learning loss, where $h$ is the discrete step and $p$ is the order of the numerical integrator.
    \item Neural ODE using non-symplectic numerical integration fail to learn conservation laws theoretically.
\end{itemize}

\subsection{Related Works}
Continuous models combining ODE and neural networks have a long history. They had already been developed and implemented to learn hidden dynamics decades ago \cite{anderson1996comparison,gonzalez1998identification,rico1994continuous,rico1993continuous}. Recently, these insights have again attracted more and more attention. The connection between dynamical systems and deep neural networks was studied in \cite{e2017proposal,e2019mean,li2017maximum,sonoda2019transport}. Neural ODE were proposed as a continuous approximation of the ResNets architecture in \cite{chen2018neural}. With their successful applications across diverse scientific disciplines \cite{botev2021priors,rubanova2019latent,yildiz2019ode2vae}, there have been extensive works studying this learning model in terms of optimization algorithms \cite{xia2021heavy, zhuang2020adaptive}, approximation capabilities \cite{zhang2020approximation}, robustness properties \cite{yan2020on}, augmentation strategies \cite{dupont2019augmented,massaroli2020dissecting} and variant architectures \cite{jia2019neural,norcliffe2021neural,zhang2019anode}.

This work concentrates on numerical integration in Neural ODE. Pal et al. \yrcite{pal2021opening} proposed a novel regularization for Neural ODE based on the internal cost of the numerical integration. Poli et al. \yrcite{poli2020hypersolver} explored the interplay between Neural ODE and numerical integration, introduced hypersolvers for fast inference. Based on numerical analysis theory, Krishnapriyan et al. \yrcite{krishnapriyan2022learning} developed a convergence test for selecting solver that makes the Neural ODE learn meaningfully continuous dynamics. As reported in \cite{gusak2020towards,zhuang2020adaptive}, changing the numerical solver yields performance degradation. Ott et al. \yrcite{ott2021resnet} and Queiruga et al. \yrcite{queiruga2020continuous} discussed this problem further. They observed that if training using coarse discretization, then testing using another solver of equal or smaller accuracy results in significantly lower performance. Thus, they proposed using a more accurate solver (adjusting the step size and using high order solver, respectively). Due to the discretization error, more accurate solver can only loosen this issue. The aim of our work is to clarify the influence of the numerical integration employed in Neural ODE models mathematically.

Neural ODE can be utilized as a data-driven technique for the discovery of latent dynamics \cite{botev2021priors, huh2020time, raissi2018multistep}. For this task, Keller and Du \yrcite{keller2021discovery} and Du et al. \yrcite{du2021discovery} provided convergence and stability analysis for Linear Multistep Neural Networks (LMNets) \cite{raissi2018multistep}. They proved that the grid error of LMNets is bounded by the sum of discretization error and approximation error under auxiliary initial conditions. Different from their works, the IMDE proposed in this paper provides a framework for the error analysis of Neural ODE for learning dynamical systems. As shown in \cref{thm:imdeerror}, our error bound is similar to theirs but our results can be applied to more general Neural ODE models.

Modified differential equations (MDE) are well-established tools for numerical analysis of solving ODE \cite{eirola1993aspects,feng1993formal,sanz1992symplectic,yoshida1993recent}. In the deep learning community, Lu et al. \yrcite{lu2018beyond} used the concept of MDE to justify the performance boost of the proposed models. Fran{\c{c}}a et al. \yrcite{fran2021on} employed MDE to study the fine behavior of gradient-based optimization. In this paper, the proposed IMDE is inspired by MDE and modified integrator \cite{chartier2007numerical} and our arguments rely on fundamental techniques of MDE \cite{feng1991formal,hairer1997life,reich1999backward}.

\section{Preliminaries}\label{sec:pre}
We begin with some notations. Consider autonomous systems of first-order ordinary differential equations
\begin{equation}\label{eq:ODE}
\frac{d}{dt}y(t) = f(y(t)),\quad y(0)=x,
\end{equation}
where $y(t) \in \R^D$, $f:\R^D \rightarrow \R^D$ is smooth and $x$ is the initial value. A non-autonomous system $\frac{d}{dt}y(t) = f(y(t),t)$ can be brought into this form by adding the variable $y_{D+1}=t$ to $y(t)$ and appending the equation $\frac{d}{dt}t = 1$. For fixed $t$, $y(t)$ can be regarded as a function of its initial value $x$. We denote
\begin{equation*}
\phi_{t}(x):=y(t) = x + \int_0^t f(y(\tau))d\tau,
\end{equation*}
which is known as the time-$t$ flow map of dynamical system (\ref{eq:ODE}). In general, we chose a numerical integrator $\Phi_{h}$ that approaches $\phi_{h}$ and compose it to obtain the numerical solution. A common choice of the numerical integrator is the Runge-Kutta method: 
\begin{equation}\label{runge-kutta}
\begin{aligned}
&v_i =x+h\sum_{j=1}^Ia_{ij}f(v_j),\quad i=1,\cdots ,I,\\
&\Phi_{h}(x) = x+h\sum_{i=1}^Ib_if(v_i),
\end{aligned}
\end{equation}
where $x$ is the initial value, $h$ is the discrete step. The coefficients $a_{ij}$, $b_i$ with $i,j=1,\cdots ,I$ fully characterize the method. In order to emphasize specific differential equation, we will add the subscript $f$ and denote $\phi_t$ as $\phi_{t,f}$ and $\Phi_h$ as $\Phi_{h,f}$ .

\subsection{Neural ODE}
Neural Ordinary Differential Equations (Neural ODE) \cite{chen2018neural} are continuous models by embedding neural networks into continuous dynamical systems. In this work, we consider the empirical risk optimization problem
\begin{equation*}
L = \frac{1}{N}\sum_{n=1}^N l(\phi_{T, f_{\theta}}(x_n), z_n),
\end{equation*}
where $\{(x_n,z_n)\}_{n=1}^N$ is the sampled training data, $l(\cdot,\cdot)$ is a loss function that is minimized when its two arguments are equal. $\phi_{T, f_{\theta}}$ is a Neural ODE model with a trainable neural network $f_{\theta}$\footnote{Under this form, time dependence can be added according to
$$\hat{x}_n = (x_n,0),\ \hat{z}_n = (z_n, T),\ \hat{f}_{\theta}= (f_{\theta},1).$$}.
Depending on the application, input or output layers are employed but we concentrate on the hidden state of the ODE layer in this paper. Exact evaluating $\phi_{T, f_{\theta}}$ is intractable and we have to use an ODE solver to approximate $\phi_{T, f_{\theta}}$. Dividing $T$ in $S$ equally-spaced intervals, the $\phi_{T, f_{\theta}}$ can be approximated by $S$ compositions of a predetermined one-step numerical integrator $\Phi_h$ (e.g. Runge-Kutta method (\ref{runge-kutta})),
\begin{equation*}
\begin{aligned}
\phi_{T, f_{\theta}} \approx &\underbrace {\Phi_{h,f_{\theta}} \circ \cdots \circ \Phi_{h,f_{\theta}}}_{\text{ $S$ compositions}}(x)\\
=& \brac{\Phi_{h,f_{\theta}}}^S(x),
\end{aligned}
\end{equation*}
where $h=T/S$ is the discrete step. Therefore, the practical input of loss function is given by the predetermined ODE solver, i.e.,
\begin{equation*}
L = \frac{1}{N}\sum_{n=1}^N l\brac{\brac{\Phi_{h,f_{\theta}}}^S(x_n), z_n}.
\end{equation*}

\section{Main Results}
Throughout this section we assume that there exists a true (but inaccessible) ODE solution such that $z_n = \phi_{T,f}(x_n)$. If Neural ODE model is employed due to the ODE dynamic benefits (e.g., improving robustness \cite{yan2020on} or ODE prior \cite{botev2021priors}), it is essential that the assumption holds and the trained model is an approximation of the true ODE. With this assumption, we are able to clarify the influence of the numerical integration on training Neural ODE models by studying the change of approximation target.

\subsection{Inverse Modified Differential Equations}\label{sec:imde}
We aim to find a perturbed differential equation
\begin{equation}\label{eq:imde}
\begin{aligned}
\frac{d}{dt}\tilde{y}(t)=&f_h(\tilde{y}(t))\\
=& f_0(\tilde{y})+hf_1(\tilde{y})+h^2f_2(\tilde{y})+\cdots,
\end{aligned}
\end{equation}
such that $\Phi_{h,f_h}(x) = \phi_{h,f}(x)$ formally. Here, identity is understood in the sense of the formal power series in $h$ without taking care of convergence issues of \cref{eq:imde}.

To obtain $f_h$, we first expand $\phi_{h,f}(x)$ into a Taylor series around $h = 0$,
\begin{equation}\label{eq:exasolu}
\begin{aligned}
\phi_{h,f}(x)=&x+hf(x)+\frac{h^2}{2}f'f(x)\\
&+\frac{h^3}{6}(f''(f,f)(x)+f'f'f(x))+\cdots .
\end{aligned}
\end{equation}
Here, the notation $f'(x)$ is a linear map (the Jacobian), the second order derivative $f''(x)$ is a symmetric bilinear map, and similarly for higher order derivatives described as a tensor. A general expansion formula for (\ref{eq:exasolu}) is given in \cref{sec:Expanding Exact Solution}.

Next, the numerical solution can be expanded as
\begin{equation}\label{eq:numsolu}
\Phi_{h,f_h}(x) = x + hd_{1,f_h}(x) + h^2d_{2,f_h}(x) + \cdots,
\end{equation}
where the functions $d_{j,f_h}$ are given and typically composed of $f_h$ and its derivatives. Expansion formulas for Runge-Kutta methods are given in \cref{sec:Expanding Numerical Solution}. For consistent integrators\footnote{An integrator is consistent if its order is not less than $1$. A Runge-Kutta method (\ref{runge-kutta}) is consistent if $\sum_{i=1}^Ib_i=1$.},
\begin{equation*}
d_{1,f_h}(x) = f_h(x) = f_0(x)+hf_1(x)+h^2f_2(x)+\cdots.
\end{equation*}

\begin{figure*}[ht]
    \begin{center}
    \centerline{\includegraphics[width=\linewidth]{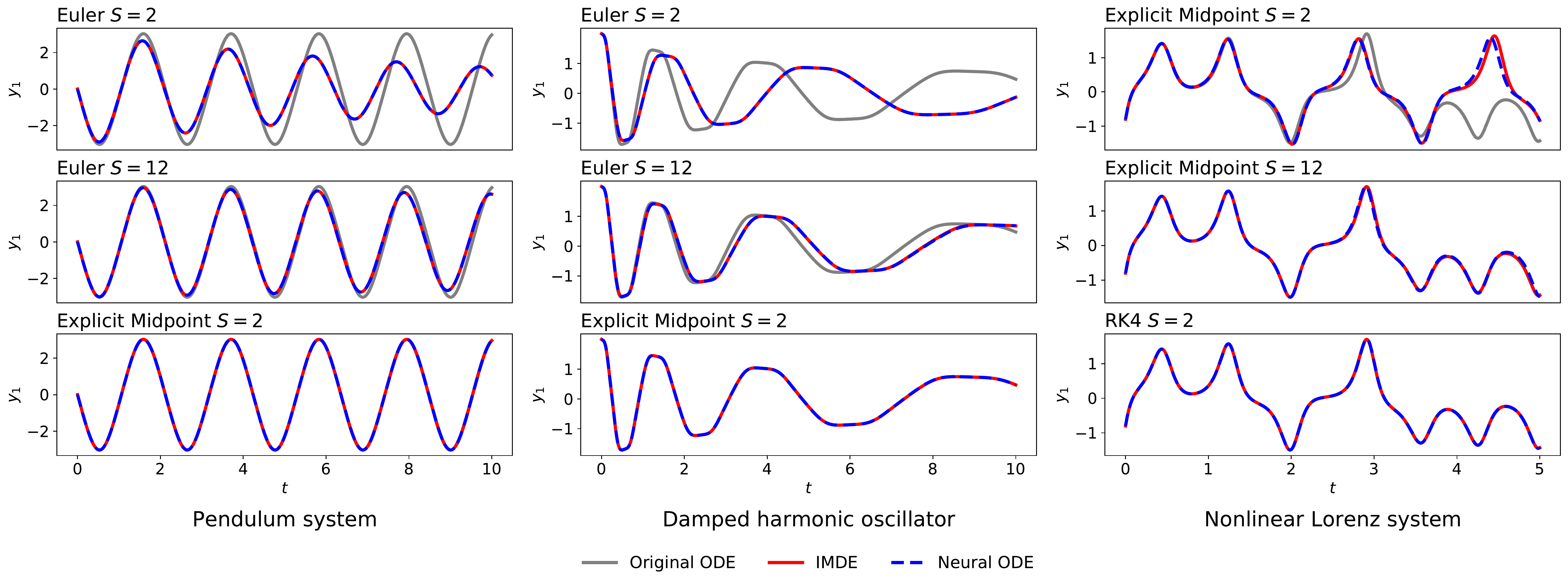}}
    \vskip -0.1in
    \caption{The first component of the trajectories of original, modified and learned equations. The Neural ODE models are trained on three tasks where the used numerical solvers are $S$ compositions of the Euler method (of order 1), the explicit midpoint rule (of order 2) and the RK4 method (of order 4), respectively. The trained Neural ODE capture the evolution of the corresponding IMDE more accurately. With the improvement of the accuracy of the solver used in training, the learned models are closer to the true systems since the discrepancy between the IMDE and the original ODE is reduced. Experimental details are presented in \cref{sec:Experimental Details}.}
    \label{fig:tras}
    \end{center}
\vskip -0.2in
\end{figure*}

In $h^id_{i,f_h}(x) $, the power of $h$ of the terms containing $f_k$ is at least $k+i$. Thus the coefficient of $h^{k+1}$ in (\ref{eq:numsolu}) is
\begin{equation*}
f_k+ \cdots,
\end{equation*}
where the ``$\cdots$'' indicates residual terms composed of $f_j$ with $j\leq k-1$ and their derivatives. By comparison of the coefficients of equal powers of $h$ in (\ref{eq:exasolu}) and (\ref{eq:numsolu}) such that these two series coincide term by term, unique functions $f_k$ in \cref{eq:imde} are obtained recursively. In \cref{sec:Two Examples for Calculating IMDE}, we present some examples illustrating the process of the above calculation. In this paper, we name the \cref{eq:imde} obtained via the above process as \textit{inverse modified differential equation (IMDE)} since it is inspired by the MDE and learning ODE is an inverse problem.

Furthermore, we obtain that formally
\begin{equation*}
\begin{aligned}
\brac{\Phi_{h,f_h}}^S(x)=\brac{\phi_{h,f}}^S(x)=\phi_{Sh,f}(x),
\end{aligned}
\end{equation*}
and the training process of Neural ODE is to minimize the difference between $\brac{\Phi_{h,f_h}}^S(x)$ and $\brac{\Phi_{h,f_{\theta}}}^S(x)$. Thus it is natural to conjecture that the trained $f_{\theta}$ is a close approximation of $f_h$. In order to substantiate this claim, we use Neural ODE to learn several benchmark problems that are widely investigated for the discovery of hidden dynamics \cite{du2021discovery, greydanus2019hamiltonian, keller2021discovery, yu2021onsagernet}. Here, the training data is generated by a known system, $\mathcal{T} = \{(x_n, \phi_{T, f}(x_n))\}_{n=1}^N$, and we can calculate the corresponding IMDE. We train the Neural ODE model using fixed step solvers with step sizes $T/S$. After training, we employ a dense numerical integration to compute the trajectories of the learned ODE and the IMDE. As displayed in \cref{fig:tras}, training Neural ODE returns approximations of the IMDE, which is consistent with the statement. We will rigorously show that this statement is true under reasonable assumptions in \cref{sec:Rigorous analysis}.

\subsection{Rigorous Analysis}\label{sec:Rigorous analysis}
In \cref{sec:imde}, we neglected the convergence issue\footnote{The series in \cref{eq:imde} does not converge in general.}. For rigorous analysis, we truncate the IMDE and denote the truncation of series in (\ref{eq:imde}) as
\begin{equation*}
f_h^K(y) = \sum_{k=0}^K h^k f_k(y).
\end{equation*}
We aim to derive an error bound between the trained $f_{\theta}$ and $f_h^{K}$ in this subsection to circumvent the convergence issue. To begin with, we introduce some notations. For a compact subset $\mathcal{K} \subset \C^D$, let $\mathcal{B}(x,r) \subset \C^D$ be the complex ball of radius $r>0$ centered at $x\in\C^D$ and let
\begin{equation*}
\mathcal{B}(\mathcal{K}, r) = \bigcup_{x \in \mathcal{K}} \mathcal{B}(x,r).
\end{equation*}
We will work with $l_{\infty}$- norm on $\C^D$ and denote $\norm{\cdot} = \norm{\cdot}_{\infty}$ For a analytic vector field $g$, we define
\begin{equation*}
\norm{g}_{\mathcal{K}} = \sup_{x\in\mathcal{K}}\norm{g(x)}.
\end{equation*}

Now, the main theorem is given as follows.
\begin{theorem}\label{thm:imde}
For $x \in \R^D$ and $r_1, r_2 >0$, a given ODE solver that is $S$ compositions of a Runge-Kutta method $\Phi_{h}$, we denote
\begin{equation*}
\mathcal{L} = \norm{\brac{\Phi_{h,f_{\theta}}}^S -\phi_{Sh,f}}_{\mathcal{B}(x, r_1)}/(Sh),
\end{equation*}
and suppose that the target vector field $f$ and the learned vector field $f_{\theta}$ are analytic and bounded by $m$ on $\mathcal{B}(x,r_1+r_2)$. Then, there exist integer $K=K(h)$ and constants $T_0$, $q$, $\gamma$, $c_1$ that depend on $r_1/m$, $r_2/m$, $S$ and $\Phi_h$, such that, if $0<T<T_0$,
\begin{equation*}
\begin{aligned}
&\norm{f_{\theta}(x) - f_h^K(x)} \leq c_1m e^{-\gamma/h^{1/q}} + \frac{e}{e-1} \mathcal{L},\\
\end{aligned}
\end{equation*}
where $e$ is the base of natural logarithm, $h=T/S$ and $f_h^K$ is the truncated vector field of the IMDE determined by $\Phi_{h}$ and $f$.
\end{theorem}
\begin{proof}
The proof can be found in \cref{sec:proofs}.
\end{proof}
Here, the first term, $c_1 m e^{-\gamma/h^{1/q}}$, is sub exponentially small, i.e., for any $k>0$, there exists a constant $c$ such that $c_1 m e^{-\gamma/h^{1/q}}<c h^k$. In statistical learning theory, learning error or expected risk typically refers to $\int l(f_{net}(x), z)dP(x,z)$. In this paper, it is of the form $\int \|\Phi_{h,f_{\theta}}(x)-\phi_{h,f}(x)\|_2^2dP(x)$,  which is the square of $L_2$-norm of $\Phi_{h,f_{\theta}}-\phi_{h,f}$. The $\mathcal{L}$ defined here, i.e., the second term, is the $L_{\infty}$-norm of $\Phi_{h,f_{\theta}}-\phi_{h,f}$ and thus measures the learning loss in the sense of generalization. If the learning loss converges to zero, the difference between the learned ODE and the truncated IMDE converges to near-zero. Thus we claim that the trained Neural ODE model is a close approximation of the IMDE.

\subsection{The Discrepancy between $f_{\theta}$ and $f$}
We have shown that training Neural ODE returns a close approximation of the corresponding IMDE instead of the true ODE. Although the true solution is unknown in practice and the IMDE is also inaccessible, we can quantify the discrepancy between $f_{\theta}$ and $f$ via investigating $f_h$.
\begin{theorem}\label{thm:imdeerror}
Suppose that the integrator $\Phi_{h}(x)$ with discrete step $h$ is of order $p\geq 1$, more precisely,
\begin{equation*}
\Phi_{h,f}(x)=\phi_{h,f}(x)+h^{p+1}\delta_f(x)+\mathcal{O}(h^{p+2}),
\end{equation*}
where $h^{p+1}\delta_f(x)$ is the leading term of the local truncation applied to (\ref{eq:ODE}). Then, the IMDE obeys
\begin{equation*}
\frac{d}{dt}\tilde{y}=f_h(\tilde{y})=f(\tilde{y})+h^pf_{p}(\tilde{y})+\cdots,
\end{equation*}
where $f_{p}(y)=-\delta_f(y)$, i.e., $f_h = f+ \mathcal{O}(h^p)$.

Furthermore, under the notations and conditions of \cref{thm:imde}, there exists a constant $c_2$ that depends on $r_1/m$, $r_2/m$, $S$ and $\Phi_h$, such that,
\begin{equation*}
\norm{f_{\theta}(x) - f(x)}\leq c_2mh^p + \frac{e}{e-1}\mathcal{L}.
\end{equation*}
\end{theorem}
\begin{proof}
The proof can be found in \cref{sec:proofs}.
\end{proof}

\begin{figure}[ht]
    \begin{center}
    \centerline{\includegraphics[width=\linewidth]{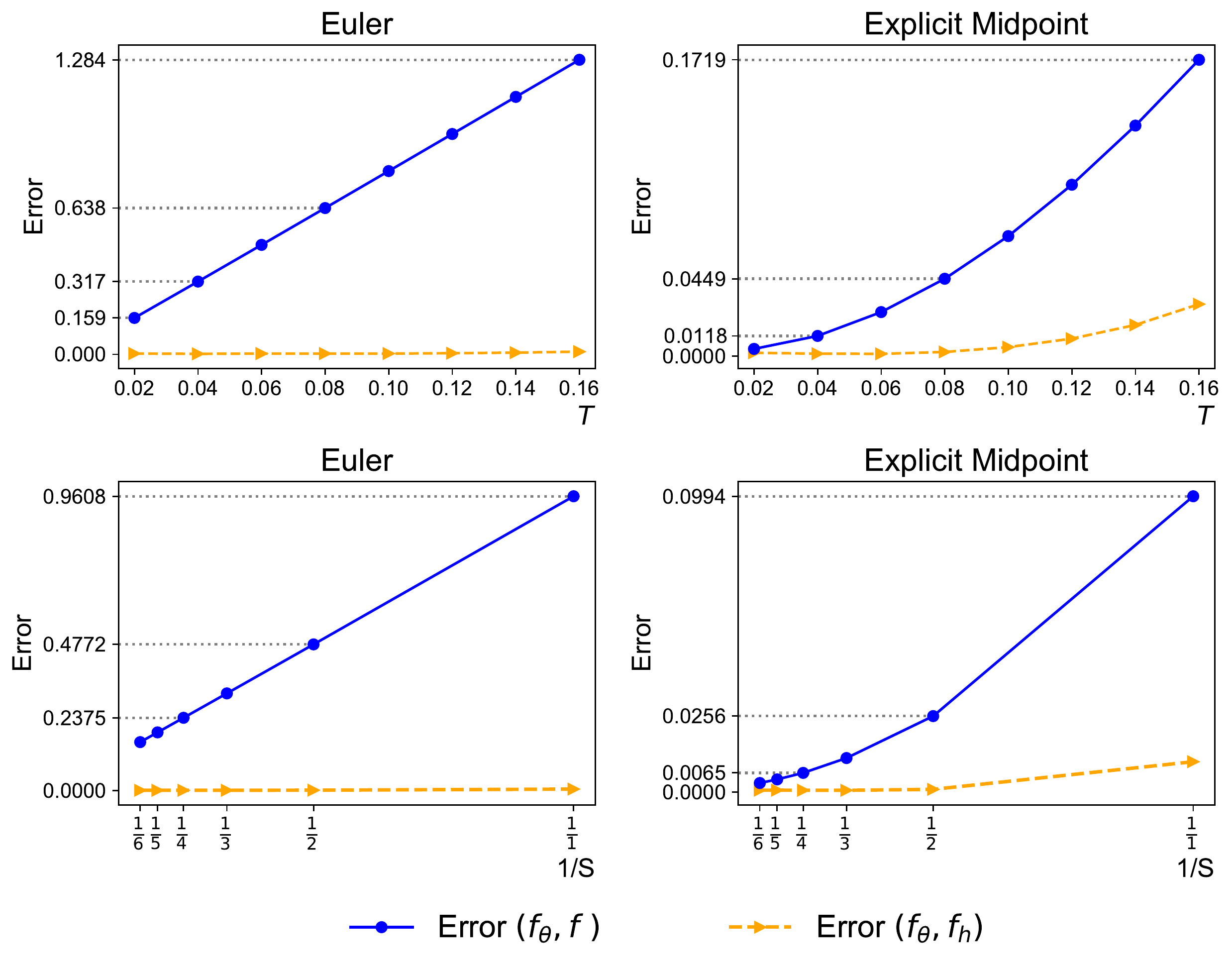}}
    \vskip -0.15in
    \caption{Error versus $h$ for learning the pendulum system. Here, the investigated numerical integrators are the Euler method (of order 1) and the explicit midpoint rule (of order 2). On top row, composition number $S$ is fixed to 1 thus $h=T$. On bottom row, data step $T$ is fixed to $0.12$ thus $h=0.12/S$. The error between $f$ and trained $f_{\theta}$ with respect to $h$ increase linearly for the Euler method and superlinearly for the explicit midpoint rule. More precisely, the error order is $1.01$ for the Euler method and $1.97$ for the explicit midpoint rule. Experimental details are presented in \cref{sec:Experimental Details}.}
    \label{fig:error}
    \end{center}
\vskip -0.2in
\end{figure}
For completeness, \cref{thm:imdeerror} was experimentally verified in \cref{fig:error}, where the error orders are consistent with the theoretical analysis.

As a direct consequence of the well-known Gr{\"o}nwall's Inequality \cite{howard1998gronwall}, we have the following corollary to provide an error bound between trajectories.
\begin{corollary}\label{cor:trajerror}
Under notations and conditions of \cref{thm:imde}, let $V_t = \{\phi_{\tau, f}(x)|0\leq \tau \leq t\}$ and 
\begin{equation*}
\mathcal{L}_t = \norm{\brac{\Phi_{h,f_{\theta}}}^S -\phi_{Sh,f}}_{\mathcal{B}(V_t, r_1)}/(Sh).    
\end{equation*} 
Then, there exist constants $C_1$, $C_2$ such that
\begin{equation*}
\norm{\phi_{t,f} (x) - \phi_{t,f_{\theta}}(x)}\leq  \frac{C_2e^{C_1t}-C_2}{C_1}(h^p+ \mathcal{L}_{t}).
\end{equation*}
\end{corollary}
\begin{proof}
The proof can be found in \cref{sec:trajerror}.
\end{proof}

According to \cref{thm:imdeerror}, if training using a coarse discretization (i.e., large $h^{p+1}\delta_f(y)$), then testing using another solver will suffer a numerical error of $\mathcal{O}(h^{p})$ besides the learning error. And only if training using sufficiently fine discretization, the trained model has the true ODE interpretation. It is worth mentioning that if Neural ODE are used for the discovery of latent dynamics, \cref{thm:imdeerror} provides an upper error bound.

\subsection{Learning Hamiltonian Systems}\label{sec:Learning Hamiltonian Systems}
Hamiltonian system is an important category in ordinary differential equations and there have been satisfactory works on learning Hamiltonian systems \cite{bertalan2019learning,chen2021data,greydanus2019hamiltonian, jin2020sympnets}. Greydanus et al. \yrcite{greydanus2019hamiltonian} observed drifting of the predicted trajectory when learning a Hamiltonian system using Neural ODE. This observation can be illuminated by IMDE.

A Hamiltonian system is formulated as
\begin{equation*}
\frac{d}{dt}y = J^{-1} \nabla H(y),\quad J=\begin{pmatrix} 0 & I \\ -I & 0\end{pmatrix},
\end{equation*}
where $I$ is $D/2$-by-$D/2$ identity matrix. As discussed above, there exists an IMDE such that formally
\begin{equation*}
\phi_{h,J^{-1} \nabla H}= \Phi_{h,f_h}.
\end{equation*}
And training Neural ODE returns an approximation of this IMDE. Therefore, learning Hamiltonian systems, or conservation laws, requires the IMDE to be a Hamiltonian system, i.e., $J f_h$ is a potential field. This is true only when the numerical integrator used in Neural ODE is symplectic.
\begin{lemma}\label{lem:imdeh}
Suppose the true system is a Hamiltonian system. If the employed numerical integrator $\Phi_h$ is symplectic, then its IMDE is also a Hamiltonian system, i.e., there locally exist smooth functions $H_k$, $k=0,1,2,\cdots$, such that
\begin{equation*}
f_k(y)=J^{-1}\nabla H_k(y).
\end{equation*}
If the employed numerical integrator $\Phi_h$ is not symplectic, then its IMDE is not a Hamiltonian system, i.e., there exists a constant $\hat{k}$ that only depends on $\Phi_h$ such that $J f_{\hat{k}}$ is not a potential field.
\end{lemma}
\begin{proof}
See \cref{sec:imdeh} for the basic concepts of symplectic integration and the detailed proof of this lemma.
\end{proof}
\begin{figure}[ht]
\vskip -0.1in
    \begin{center}
    \centerline{\includegraphics[width=\linewidth]{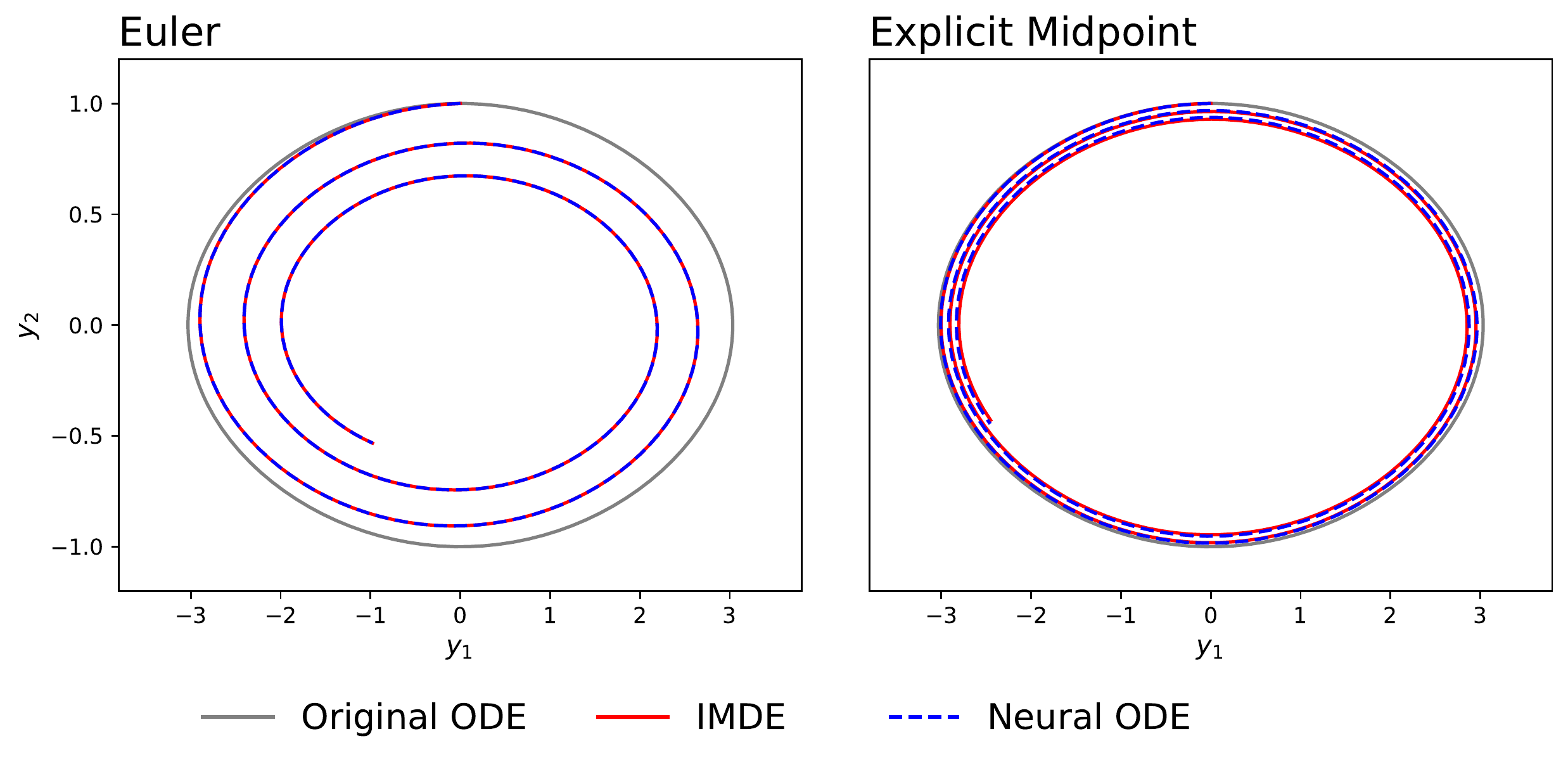}}
    \vskip -0.15in
    \caption{Learning pendulum system using Neural ODE. The dynamics of learned models gradually drift away from the ground truth and accurately match the IMDE. Experimental details are presented in \cref{sec:Experimental Details}.}
    \label{fig:hami}
    \end{center}
    \vskip -0.2in
\end{figure}
A non-symplectic numerical integrator cannot guarantee that its IMDE is always a Hamiltonian system. Thus Neural ODE using non-symplectic integration fail to learn conservation laws. \cref{fig:hami} validates this statement experimentally. We remark that any explicit Runge-Kutta method\footnote{A Runge-Kutta method (\ref{runge-kutta}) is explicit if $a_{ij}=0$ for $i\leq j$ and implicit otherwise. An implicit method has to be calculated iteratively, thus it is not employed in vanilla Neural ODE.} can not be symplectic (see e.g., Theorem \uppercase\expandafter{\romannumeral 6}.4.3 of \cite{hairer2006geometric}).

\subsection{Discussion on Conditions in \cref{thm:imde}}
The generalization requirement, i.e., using $\mathcal{L}$ as an error bound, is in some sense necessary. Otherwise, if the ODE solver is one composition of the implicit Euler method\footnote{\begin{equation*}
v_1 =x+hf(v_1),\quad \Phi_{h}(x) = x+hf(v_1).
\end{equation*}}, then, the learning model has no information at $x$. Neural network models have better generalization in practice, thus small $\mathcal{L}$ can be attained and it is reasonable to use $\mathcal{L}$ as a part of the error bound.

The analyticity and boundness requirements are the fundamental assumptions for the estimates in \cref{thm:imde}, which allow us to use complex analysis techniques to complete the proof. To illustrate their reasonableness, we consider the following two learning tasks.

\begin{example}\label{example1}
Consider learning the differential equation
\begin{equation*}
\begin{aligned}
&\frac{d}{dt}p=1, \\
&\frac{d}{dt}q=\sin{(\frac{2\pi}{h}p+b)},
\end{aligned}
\end{equation*}
with parameter $b$ and initial value $(p(0),q(0))=(p_0,q_0)$. The exact solution is given as
\begin{equation*}
\begin{aligned}
&p(t)=p_0+t, \\
&q(t)=q_0-\frac{h}{2\pi}(\cos{(\frac{2\pi}{h}(p_0+t)+b)}-\cos{(\frac{2\pi}{h}p_0+b)}).
\end{aligned}
\end{equation*}
Taking $t=h$, we have that
\begin{equation*}
\begin{aligned}
&p(h)=p_0+h, \\
&q(h)=q_0.
\end{aligned}
\end{equation*}
Thus, same exact solutions are obtained although the parameter $b$ is different, which yields multiple ODE interpretation of training data.
\end{example}

\begin{example}\label{example2}
Consider learning linear equation
\begin{equation*}
\frac{d}{dt} p = \lambda p
\end{equation*}
with parameter $\lambda$. Applying the explicit Euler method twice yields
\begin{equation*}
(\Phi_h)^2(p_0) = (1+\lambda h)^2 p_0 = (1+ (-2/h-\lambda)h)^2p_0.
\end{equation*}
Same numerical solutions are obtained for parameters $\lambda$ and $(-2/h-\lambda)$, indicating that $f_{\theta}$ can approach different targets to minimize loss.
\end{example}

The analyticity and boundness requirements indicate the boundness of derivatives of $f$ and $f_{\theta}$ due to Cauchy's estimate in several variables (see e.g., Section 1.3 of \cite{scheidemann2005introduction}), more precisely,
\begin{equation*}
\|D^{\bm{\alpha}} f\|_{\mathcal{B}(x,r_1)},\ \|D^{\bm{\alpha}}f_{\theta}\|_{\mathcal{B}(x,r_1)} \leq \bm{\alpha}!m r_2^{-|\bm{\alpha}|}.
\end{equation*}
The requirement for the true ODE excludes counterexamples similar to those in \cref{example1} and implies that our results only hold for low-frequency true ODE interpretation. In addition, the assumption for the learned ODE excludes high-frequency neural network solutions similar to those in \cref{example2}. When training Neural ODE, we can enforce the weights for each layer of $f_{\theta}$ to attain low-frequency assumption of the learned ODE. In classical regression problems, training FNN first captures low-frequency components of the target function and then approximates the high-frequency \cite{luo2019theory,xu2019training}. We conjecture that the implicit regularization is also applied to Neural ODE, and thus the analyticity and boundness assumption of $f_{\theta}$ holds without any explicit regularization.

Experimental results in \cref{fig:tras} validate both assumptions, where Neural ODE capture the evolutions of the IMDE without regularization. It is worth mentioning that the assumptions of generalization and analyticity are also required in the global error analysis of GFNN \cite{chen2021data}, and their conclusion is also confirmed experimentally.

The requirement of the Runge-Kutta method is not necessary. \cref{thm:imde} holds for any numerical integrator satisfying the following assumption.
\begin{assumption}\label{asm:int}
For analytic $g$, $\hat{g}$ satisfying $\norm{g}_{\mathcal{B}(\mathcal{K},r)}\leq m$, $\norm{\hat{g}}_{\mathcal{B}(\mathcal{K},r)} \leq m$, there exist constants $b_1,b_2,b_3$ that depend only on the $\Phi_{h}$ and $S$ such that
\begin{itemize}
    \item For $|h|\leq h_0 = b_{1} r/m$ , $\brac{\Phi_{h,\hat{g}}}^S$, $\brac{\Phi_{h,g}}^S$ are analytic on $\mathcal{K}$.
    \item for $|h|\leq h_0$,
    \begin{equation*}
        \norm{\brac{\Phi_{h,\hat{g}}}^S-\brac{\Phi_{h,g}}^S}_{\mathcal{K}}\leq b_{2}h\norm{\hat{g}-g}_{\mathcal{B}(\mathcal{K},r)}.
    \end{equation*}
    \item for $|h|< h_1 < h_0$,
    \begin{equation*}
    \begin{aligned}
    \norm{\hat{g}-g}_{\mathcal{K}} \leq & \frac{1}{S|h|}\norm{\brac{\Phi_{h,\hat{g}}}^S-\brac{\Phi_{h,g}}^S}_{\mathcal{K}} \\
    & + \frac{b_2|h|}{h_1-|h|} \norm{\hat{g}-g}_{\mathcal{B}(\mathcal{K},b_3h_1 m)}.
    \end{aligned}
    \end{equation*}
\end{itemize}
\end{assumption}
We will show that Runge-Kutta methods satisfy \cref{asm:int} in \cref{sec:Properties of Runge-Kutta methods}.

\section{Conclusion and Limitations} \label{sec:conclusion and Limitations}
In this paper, we perform numerical analysis on the numerical integration in Neural ODE. The main contribution is that we propose the inverse modified differential equations (IMDE) and prove that training a Neural ODE model actually returns an approximation of the IMDE determined by the learning task and the employed ODE solver, rather than the true ODE. This theorem clarifies the influence of the numerical integration on training Neural ODE models by pointing out the change of approximation target. In addition, we show that the discrepancy between the trained model and the unknown system is bounded by the sum of discretization error $Ch^p$ and learning loss, where $h$ is the discrete step and $p$ is the order of integrator. It provides an upper error bound for the discovery of hidden dynamics using Neural ODE. We also discuss learning the Hamiltonian system. IMDE reveal the potential problems, showing that Neural ODE using non-symplectic integration fail to learn conservation laws. Experimental results support the theoretical analysis.

One limitation of our work is the generalization and boundness requirements on complex space. Quantifying the generalization error and implicit regularization for supervised learning are still open research problems. In addition, the current IMDE is not applicable to adaptive step size selection. We would like to further investigate such problems for Neural ODE in the future.  

If the discrete Neural ODE is employed as a discrete model, our assumption that there exists a valid ODE solution does not hold since Neural ODE models are limited in their approximation capabilities \cite{zhang2020approximation}. In this case, the influence of the numerical integration, together with how to measure this influence, remains unknown.

Like modified differential equations, IMDE only introduce a framework of theoretical analysis of the numerical integration in Neural ODE. Based on the analysis results, improving Neural ODE from the point of view of the ODE solver might be another interesting direction.

\section*{Acknowledgements}
The authors thank the anonymous reviewers for their valuable comments. This work is supported by the Major Project on New Generation of Artificial Intelligence from MOST of China (Grant No. 2018AAA0101002), and National Natural Science Foundation of China (Grant Nos. 11901564 and 12171466).
%
%


\bibliography{main}
\bibliographystyle{icml2022}

\newpage
\appendix
\onecolumn
\section{Calculation of IMDE}\label{sec:Calculation of IMDE}
\subsection{Expanding Exact Solution by Lie Derivatives}\label{sec:Expanding Exact Solution}
Following \cite{hairer2006geometric}, we briefly introduce Lie derivatives. Given ordinary differential equations
\begin{equation*}
\frac{d}{dt}y(t) = f(y(t)),
\end{equation*}
Lie derivative $\D$ is the differential operator defined as:
\begin{equation*}
\D g(y)=g'(y)f(y), \quad g:\R^D \rightarrow \R^D.
\end{equation*}
According to the chain rule, we have
\begin{equation*}
\frac{d}{dt}g(\phi_{t,f}(x))=(\D g)(\phi_{t,f}(x)),
\end{equation*}
and thus obtain the Taylor series of $g(\phi_{t,f}(x))$ developed at $t=0$:
\begin{equation*}
g(\phi_{t,f}(x))=\sum_{k=0}^{\infty}\frac{t^k}{k!}(\D^kg)(x).
\end{equation*}
In particular, by setting $t=h$ and $g(y)=I_D(y)=y$, the identity map, it turns to the Taylor series of the exact solution $\phi_{h,f}$ itself, i.e.,
\begin{equation}\label{eq:exasolu-a}
\begin{aligned}
\phi_{h,f}(x)=&\sum_{k=0}^{\infty}\frac{h^k}{k!}(\D^kI_D)(x)=x+hf(x)+\frac{h^2}{2}f'f(x)+\frac{h^3}{6}(f''(f,f)(x)+f'f'f(x))+\cdots .
\end{aligned}
\end{equation}
Here, the notation $f'(x)$ is a linear map (the Jacobian), the second order derivative $f''(x)$ is a symmetric bilinear map and similarly for higher order derivatives described as tensor, more precisely, for $k$ vector fields $g^1, \cdots, g^k:\R^D \rightarrow \R^D$,
\begin{equation*}
f^{(k)}(g^1, \cdots, g^k) = \left(\sum_{i_1,\cdots, i_k=1}^D\frac{\partial ^k f_1}{\partial x_{i_1}\cdots \partial x_{i_k}}g^1_{i_1}\cdots g^k_{i_k},\ \cdots, \sum_{i_1,\cdots, i_k=1}^D\frac{\partial ^k f_D}{\partial x_{i_1}\cdots \partial x_{i_k}}g^1_{i_1}\cdots g^k_{i_k}\right)^{\top},
\end{equation*}
where the subscript $i$ indicates the $i$-th component.

\subsection{Expanding Runge-Kutta methods}\label{sec:Expanding Numerical Solution}
The expansion of numerical solutions has been well-developed in the last few decades \cite{hairer2006geometric}. Here, we briefly introduce the expansion of Runge-Kutta methods. Given real numbers $b_i,a_{ij}\ (i,j=1,\cdots, s)$, a Runge-Kutta method for solving $\frac{d}{dt}y(t) = g(y(t))$ with initial value $x$ is defined as
\begin{equation}\label{runge-kutta-a}
\begin{aligned}
&v_i =x+h\sum_{j=1}^Ia_{ij}g(v_j)\quad i=1,\cdots ,I,\\
&\Phi_{h}(x) = x+h\sum_{i=1}^Ib_ig(v_i).
\end{aligned}
\end{equation}
The coefficients $a_{ij}$, $b_i$ with $i,j=1,\cdots ,I$ fully characterize the method and also the expansion.
\begin{lemma}\label{the:RKexp}
The derivatives of the solution of a Runge-Kutta method (\ref{runge-kutta-a}) with respect to $x$, for $h=0$, are given by
\begin{equation}\label{eq:tree}
\frac{\partial^k \Phi_{h}(x)}{\partial h^k}\Big|_{h=0} =\sum_{|\tau|=k}\gamma(\tau) \cdot \alpha(\tau) \cdot \phi(\tau) \cdot F(\tau)(x).
\end{equation}
Here, $\tau$ is called trees and $|\tau|$ is the order of $\tau$ (the number of vertices). $\gamma(\tau)$, $\phi(\tau)$, $\alpha(\tau)$ are positive integer coefficients, $F(\tau)(y)$ is called elementary differentials and typically composed of $g(y)$ and its derivatives.
\end{lemma}

\begin{table}[t]
    \caption{Trees, coefficients and elementary differentials in \cref{eq:tree}}
    \label{tab:RKtree}
    \vskip 0.15in
    \centering
    \begin{tabular}{|c|c|c|c|c|c|}
    \hline
       $|\tau|$& $\tau$             &$\gamma(\tau)$& $\alpha(\tau)$ & $\phi(\tau)$ & $F(\tau)$ \\
     \hline
     \hline
       1&$\bullet$                  &1 &1 & $\sum_ib_i$                        & $g$           \\
     \hline
       2&$[\bullet]$                &2 &1 & $\sum_{ij}b_ia_{ij} $              & $g'g$         \\
     \hline
       3&$[\bullet,\bullet]$        &3 &1 & $\sum_{ijk}b_ia_{ij}a_{ik}$        & $g''(g,g)$    \\

       3&$[[\bullet]]$              &6 &1 & $\sum_{ijk}b_ia_{ij}a_{jk}$        & $g'g'g$       \\
     \hline
       4&$[\bullet,\bullet,\bullet]$&4 &1 & $\sum_{ijkl}b_ia_{ij}a_{ik}a_{il}$ & $g'''(g,g,g)$ \\

       4&$[[\bullet],\bullet]$      &8 &3 & $\sum_{ijkl}b_ia_{ij}a_{ik}a_{jl}$ & $g''(g'g,g)$  \\

       4&$[[\bullet,\bullet ]]$     &12&1 & $\sum_{ijkl}b_ia_{ij}a_{jk}a_{jl}$ & $g'g''(g,g)$  \\

       4&$[[[\bullet]]]$            &24&1 & $\sum_{ijkl}b_ia_{ij}a_{jk}a_{kl}$ & $g'g'g'g$     \\
     \hline
    \end{tabular}
\vskip -0.1in
\end{table}
\begin{proof}
Some $\gamma(\tau), \alpha(\tau), \phi(\tau), F(\tau)$ are reported in \cref{tab:RKtree}, detailed proof and calculations can be found in Section \uppercase\expandafter{\romannumeral3}.1 of \cite{hairer2006geometric}.
\end{proof}
Due to \cref{the:RKexp}, the formal expansion of a Runge-Kutta method is given by
\begin{equation*}
\Phi_{h,g}(x)=y+hd_{1,g}(x)+h^2d_{2,g}(x)+ \cdots,
\end{equation*}
where
\begin{equation*}
d_{k,g}(x) = \frac{1}{k!}\frac{\partial^k \Phi_{h}(x)}{\partial h^k}\Big|_{h=0}=\frac{1}{k!}\sum_{|\tau|=k}\gamma(\tau) \cdot \alpha(\tau) \cdot \phi(\tau) \cdot F(\tau)(x).
\end{equation*}

\subsection{Two Examples for Calculating IMDE}\label{sec:Two Examples for Calculating IMDE}
The next examples illustrate the process of calculation of IMDE.
\begin{example}
Consider the explicit Euler method
\begin{equation}\label{eq:numsolu-e}
\Phi_{h,f_h}(x)=x+hf_h(x) = x + h\sum_{k=0}^{\infty} h^kf_k.
\end{equation}
Here, we simply have $d_{1,f_h} = f_h$ and $d_{j,f_h}=0$ for all $j \geq 2$.

Comparing equal powers of $h$ in the expression (\ref{eq:exasolu-a}) and (\ref{eq:numsolu-e}), and setting $y:=x$ yields recurrence relations for functions $f_j$, i.e.,
\begin{equation*}
\begin{aligned}
f_0(y) =& f(y),\\
f_1(y) =& \frac{1}{2}f'f(y),\\
f_2(y) =& \frac{1}{6}(f''(f,f)(y)+f'f'f(y)),\\
f_3(y) =& \frac{1}{24}(f'''(f,f,f)(y)+3f''(f'f,f)(y)+f'f''(f,f)(y)+f'f'f'f(y)),\\
& \vdots
\end{aligned}
\end{equation*}
\end{example}

\begin{example}
The explicit midpoint rule
\begin{equation*}
v_1 = x+\frac{h}{2}f_h(x), \quad \Phi_{h,f_h}(x)=x+hf_h(v_1)
\end{equation*}
can be expanded as
\begin{equation*}
\begin{aligned}
\Phi_{h,f_h}(x)=&x+hf_h(x)+ \frac{h^2}{2}f_h'f_h(x) + \frac{h^3}{8}f_h''(f_h,f_h)(x) + \frac{h^4}{48}f_h'''(f_h,f_h,f_h)(x)+\cdots
\end{aligned}
\end{equation*}
according to \cref{the:RKexp}. Plugging $f_h=\sum_{k=0}^\infty h^kf_k$ yields
\begin{equation}\label{eq:numsolu-m}
\begin{aligned}
&\Phi_{h,f_h}(x)=x+ hf_0(x) + h^2\big(f_1(x)+ \frac{1}{2}f_0'f_0(x)\big)+h^3\big(f_2(x)+ \frac{1}{2}f_1'f_0(x)+ \frac{1}{2}f_0'f_1(x)+ \frac{1}{8}f_0''(f_0,f_0)(x)\big)\\
&+h^4\big(f_3(x)+ \frac{1}{2}f_1'f_1(x)+ \frac{1}{2}f_0'f_2(x)+\frac{1}{2}f_2'f_0(x)+\frac{1}{8}f_1''(f_0,f_0)(x)+\frac{1}{4}f_0''(f_1,f_0)(x)+\frac{1}{48}f_0'''(f_0,f_0,f_0)(x)\big)+ \cdots
\end{aligned}
\end{equation}
Comparing equal powers of $h$ in the expression (\ref{eq:exasolu-a}) and (\ref{eq:numsolu-m}), and setting $y:=x$ yields recurrence relations for functions $f_j$, viz.,
\begin{equation*}
\begin{aligned}
f_0(y) =& f(y),\\
f_1(y) =& 0,\\
f_2(y) =& \frac{1}{24}f''(f,f)(y)+\frac{1}{6} f'f'f(y),\\
f_3(y) =&-\frac{1}{16} f'f''(f,f)(y)-\frac{1}{8}f'f'f'f(y),\\
& \vdots
\end{aligned}
\end{equation*}
\end{example}
We remark that calculating IMDE is one of the steps for constructing modified integrator \cite{chartier2007numerical}, where an explicit recurrence formula based on B-series is given.

\section{Proofs}\label{sec:proofs}
The proofs rely on the definition of IMDE, the induction idea and some complex analysis techniques such as the maximum principle and Cauchy's estimate (see e.g., \cite{burckel1980introduction}).

\subsection{Properties of IMDE}
The ODE solver, i.e., $S$ compositions of an integrator $\Phi_{h}$, can be regarded as a one-step integrator with discrete step $Sh$ and thus has its IMDE. The following lemma indicates that the IMDE of the ODE solver coincides with the IMDE of $\Phi_h$.
\begin{lemma}\label{lem:inmde}
For any fixed composition number $S$, suppose that the vector fields of the IMDE of $\Phi_{h}$ and $\brac{\Phi_{h}}^S$ are $f_{h}(y) = \sum_{k=0}^{\infty} h^k f_k(y)$ and $F_{Sh}(y) = \sum_{k=0}^{\infty} (Sh)^k F_k(y)$, respectively. Then, for any integer $K$, $f_{h}^K=F_{Sh}^K$.
\end{lemma}
\begin{proof}
The calculation procedure of $f_h$ uniquely defines the functions $f_k$ and can be rewritten as the following recursion:
\begin{equation}\label{eq:recursion}
h^{k+1}f_{k}= \phi_{h,f} - \Phi_{h,f_h^{k-1}} +\mathcal{O}(h^{k+2}).
\end{equation}
We first prove
\begin{equation}\label{eq:Srecursion}
Sh^{k+1}f_k = \phi_{Sh,f} - \brac{\Phi_{h,f_{h}^{k-1}}}^S+ \mathcal{O}(h^{k+2}),
\end{equation}
by induction on $S \geq 1$. First, the case when $S=1$ is obvious. Suppose now that the statement holds for $S-1$. Then, by this inductive hypothesis, we obtain
\begin{equation*}
\begin{aligned}
\phi_{Sh,f} - \brac{\Phi_{h,f_{h}^{k-1}}}^S
=& \Phi_{h,f_{h}^{k-1}} \circ\brac{\phi_{(S-1)h,f} - \brac{\Phi_{h,f_{h}^{k-1}}}^{S-1}} + \brac{\phi_{h,f} - \Phi_{h,f_{h}^{k-1}}}\circ \phi_{(S-1)h,f} \\
=& (S-1)h^{k+1} \Phi_{h,f_{h}^{k-1}} \circ f_k + h^{k+1}f_k \circ \phi_{(S-1)h,f} + \mathcal{O}(h^{k+2})\\
=&Sh^{k+1}f_k + \mathcal{O}(h^{k+2}),
\end{aligned}
\end{equation*}
where we have used the fact that \begin{equation*}
\phi_{(S-1)h,f} = I_D + \mathcal{O}\left((S-1)h\right),\ \Phi_{h,f_{h}^{k-1}} = I_D + \mathcal{O}(h).
\end{equation*}
Hence the induction is completed.

We next prove that $S^kF_k = f_k$ by induction on $k$. First the case when $k=0$ is obvious since $F_0 = f_0=f$. Suppose now $S^kF_k = f_k$ holds for $k\leq K-1$. This inductive hypothesis implies that $F_{Sh}^{K-1}=f_h^{K-1}$. Using (\ref{eq:recursion}) for $F_K$ we obtain
\begin{equation*}
(Sh)^{K+1}F_K = \phi_{Sh,f} - \brac{\Phi_{h,F_{Sh}^{K-1}}}^S+ \mathcal{O}(h^{K+2}) = \phi_{Sh,f} - \brac{\Phi_{h,f_{h}^{K-1}}}^S+ \mathcal{O}(h^{K+2}).
\end{equation*}
This together with (\ref{eq:Srecursion}) concludes the induction and thus completes the proof.
\end{proof}

\begin{lemma}\label{lem:modiode}
Suppose that
\begin{equation*}
\Phi_{h,f}(x)=\phi_{h,f}(x)+h^{p+1}\delta_f(x)+\mathcal{O}(h^{p+2}).
\end{equation*}
Then, the IMDE obeys
\begin{equation*}
\frac{d}{dt}\tilde{y}=f_h(\tilde{y})=f(\tilde{y})-h^p\delta_f(\tilde{y})+\cdots.
\end{equation*}
\end{lemma}
\begin{proof}
We prove that $f_k=0$ for $k\leq p-1$ and $f_k=-\delta_f$ for $k=p$ by induction on $k$. By \cref{eq:recursion}, we have
\begin{equation*}
h^{2}f_{1} = \phi_{h,f} -\Phi_{h,f} +\mathcal{O}(h^{3}) = -h^{p+1}\delta_f+\mathcal{O}(h^{3}).
\end{equation*}
Thus $f_1=0$ if $1\leq p-1$ and $f_1 = -\delta_f$ if $1=p$. Suppose now the the hypothesis holds for $k\leq K-1 <p$.
The function $f_{K}$ is obtained from
\begin{equation*}
h^{K+1}f_{K} = \phi_{h,f} -\Phi_{h,f_h^{K-1}} +\mathcal{O}(h^{K+2}) = \phi_{h,f} -\Phi_{h,f} +\mathcal{O}(h^{K+2})=-h^{p+1}\delta_f+\mathcal{O}(h^{K+2}).
\end{equation*}
Thus $f_{K}=0$ if $K\leq p-1$ and $f_{K} = -\delta_f$ if $K=p$. The proof is completed.
\end{proof}

\subsection{Properties of Runge-Kutta Methods}\label{sec:Properties of Runge-Kutta methods}
We first consider the case $S=1$ in \cref{asm:int}, i.e.,
\begin{equation}\label{con:intbound}
\begin{aligned}
&1. \text{\ Analytic for}\ |h|\leq h_0 = b_{1} r/m\ \text{and}\ y \in \mathcal{K},\\
&2. \norm{\Phi_{h,\hat{g}}-\Phi_{h,g}}_{\mathcal{K}}\leq b_{2}h\norm{\hat{g}-g}_{\mathcal{B}(\mathcal{K},r)}, \ \text{for}\ |h|\leq h_0,\\
&3. \norm{\hat{g}-g}_{\mathcal{K}} \leq \frac{1}{|h|}\norm{\Phi_{h,\hat{g}} -\Phi_{h,g}}_{\mathcal{K}} + \frac{b_2|h|}{h_1-|h|} \norm{\hat{g}-g}_{\mathcal{B}(\mathcal{K},b_3h_1 m)}, \ \text{for}\ |h|< h_1 \leq h_0.
\end{aligned}
\end{equation}
And we prove that the condition (\ref{con:intbound}) is satisfied for Runge-Kutta methods (\ref{runge-kutta}).
\begin{lemma}\label{lem:rk}
For a consistent Runge-Kutta method (\ref{runge-kutta}) denoted as $\Phi_{h}$, let
\begin{equation*}
\mu = \sum_{i=1}^I|b_i|, \quad \kappa =\max_{1\leq i\leq s}\sum_{j=1}^I|a_{ij}|.
\end{equation*}
Consider analytic $g$, $\hat{g}$ satisfying $\norm{g}_{\mathcal{B}(\mathcal{K},r)}\leq m$, $\norm{\hat{g}}_{\mathcal{B}(\mathcal{K},r)} \leq m$, if $\kappa \neq0$, then $\Phi_{h,g}$, $\Phi_{h,\hat{g}}$ are analytic for $|h|\leq h_0=r/(4\kappa m)$ and
\begin{equation*}
\norm{\Phi_{h,\hat{g}}-\Phi_{h,g}}_{\mathcal{K}}\leq 2\mu |h|\norm{\hat{g}-g}_{\mathcal{B}(\mathcal{K},|h|\kappa m)}.
\end{equation*}
Furthermore, for $|h|< h_1 \leq h_0$,
\begin{equation*}
\norm{\hat{g}-g}_{\mathcal{K}} \leq \frac{\norm{\Phi_{h,\hat{g}}-\Phi_{h,g}}_{\mathcal{K}}}{|h|} +\frac{2\mu|h|\norm{\hat{g}-g}_{\mathcal{B}(\mathcal{K},h_1\kappa m)}}{h_1-|h|} .
\end{equation*}
\end{lemma}
\begin{proof}
For $y \in \mathcal{B}(\mathcal{K}, r/2)$ and $\norm{\Delta y} \leq 1$, the function $\alpha(z) = g(y+ z \Delta y)$ is analytic for $|z| \leq r/2$ and bounded by $m$. By Cauchy's estimate, we obtain
\begin{equation*}
\norm{g'(y)\Delta y} = \norm{\alpha ' (0)} \leq 2m/r,
\end{equation*}
and $\norm{g'(y)}\leq 2m/r$ for $y \in \mathcal{B}(\mathcal{K}, r/2)$ in the operator norm.

For a Runge-Kutta method (\ref{runge-kutta}) with initial point $x\in \mathcal{K}$, the solution can be obtained by the nonlinear systems
\begin{equation*}
\begin{aligned}
&u_i =x+h\sum_{j=1}^I a_{ij}\hat{g}(u_j)\quad i=1,\cdots ,I,\quad \Phi_{h,\hat{g}}(x) = x+h\sum_{i=1}^I b_i\hat{g}(u_i),\\
&v_i =x+h\sum_{j=1}^I a_{ij}g(v_j)\quad i=1,\cdots ,I,\quad \Phi_{h,g}(x) = x+h\sum_{i=1}^I b_ig(v_i).
\end{aligned}
\end{equation*}
Due to the Implicit Function Theorem \cite{scheidemann2005introduction}, $u_i,v_i$ possess unique solutions on the closed set $\mathcal{B}(\mathcal{K},|h|\kappa m)$ if $2|h|\kappa m/r \leq \gamma<1$ and the method is analytic for $|h|\leq \gamma r/ 2\kappa m$.

In addition,
\begin{equation*}
\begin{aligned}
\norm{u_i-v_i}
\leq & |h|\sum_{j=1}^I|a_{ij}|(\norm{\hat{g}(u_j)-\hat{g}(v_j)}+ \norm{\hat{g}(v_j)-g(v_j)})\\
\leq & |h|\kappa \frac{2m}{r}\max_{1\leq j\leq I}\norm{u_j-v_j}+|h|\kappa\norm{\hat{g}-g}_{\mathcal{B}(\mathcal{K},|h|\kappa m)}.
\end{aligned}
\end{equation*}
Thus we obtain
\begin{equation*}
\max_{1\leq i\leq I}\norm{u_i-v_i} \leq \frac{\kappa}{1-|h|\kappa \frac{2m}{r}} |h|\norm{\hat{g}-g}_{\mathcal{B}(\mathcal{K},|h|\kappa m)}.
\end{equation*}
Next, we have
\begin{equation*}
\begin{aligned}
\norm{\Phi_{h,\hat{g}}(x)-\Phi_{h,g}(x)}
\leq& |h|\sum_{i=1}^I|b_{i}|\norm{\hat{g}(u_i)-\hat{g}(v_i)}+ |h|\sum_{i=1}^I|b_{i}|\norm{\hat{g}(v_i)-g(v_i)}\\
\leq& |h|\mu \frac{2m}{r}\max_{1\leq j\leq s}\norm{u_j-v_j}+|h|\mu \norm{\hat{g}-g}_{\mathcal{B}(\mathcal{K},|h|\kappa m)}\\
\leq& \left(|h|\mu \frac{2m}{r}\frac{\kappa}{1-|h|\kappa \frac{2m}{r}} +\mu \right) |h|\norm{\hat{g}-g}_{\mathcal{B}(\mathcal{K},|h|\kappa m)}.
\end{aligned}
\end{equation*}
Taking $\gamma = 1/2$, together with the arbitrariness of $x$, yields
\begin{equation*}
\norm{\Phi_{h,\hat{g}}-\Phi_{h,g}}_{\mathcal{K}} \leq 2\mu |h|\norm{\hat{g}-g}_{\mathcal{B}(\mathcal{K},|h|\kappa m)}.
\end{equation*}
These complete the first part of the proof.

Finally, using Cauchy's estimate, we deduce that for $h_1 \leq h_0$,
\begin{equation*}
\begin{aligned}
\norm{\frac{d^i}{dh^i}\left(\Phi_{h,\hat{g}}(x)-\Phi_{h,g}(x)\right)\Big|_{h=0}}
\leq \frac{ i!2\mu \norm{\hat{g}-g}_{\mathcal{B}(\mathcal{K},h_1\kappa m)}}{h_1^{i-1}}.
\end{aligned}
\end{equation*}
By the analyticity and triangle inequality, we obtain for $|h|<h_1$,
\begin{equation*}
\begin{aligned}
\norm{\Phi_{h,\hat{g}}(x)-\Phi_{h,g}(x)}
\geq& |h|\norm{\hat{g}(x)-g(x)} - \sum_{i=2}^{\infty}\norm{\frac{h^i}{i!}\frac{d^j}{dh^j}\brac{\Phi_{h,\hat{g}}(x)-\Phi_{h,g}(x)}\Big|_{h=0}}\\
\geq& |h|\norm{\hat{g}(x)-g(x)} - 2\mu|h| \norm{\hat{g}-g}_{\mathcal{B}(\mathcal{K},h_1\kappa m)}\sum_{i=2}^{\infty} \left(\frac{|h|}{h_1}\right)^{i-1}. \\
\end{aligned}
\end{equation*}
Therefore,
\begin{equation*}
\norm{\hat{g}-g}_{\mathcal{K}} \leq \frac{\norm{\Phi_{h,\hat{g}}-\Phi_{h,g}}_{\mathcal{K}}}{|h|} +\frac{2\mu|h|\norm{\hat{g}-g}_{\mathcal{B}(\mathcal{K},h_1\kappa m)}}{h_1-|h|},
\end{equation*}
which concludes the proof.
\end{proof}
We could easily check that for the case $\kappa = 0$, i.e., the Euler method, condition (\ref{con:intbound}) also holds.

If we apply a Runge-Kutta method $\Phi_h$ with coefficients $a_{ij}, b_{i}$, step $h$ and initial value $x$, then apply another Runge-Kutta method $\hat{\Phi}_{\hat{h}}$ with coefficients $\hat{a}_{ij}, \hat{b}_{i}$, step $\hat{h}$ and initial value $\Phi_h(x)$, this composition of two methods can be regarded as a single Runge-Kutta method with discrete step $h+\hat{h}$:
\begin{equation*}
\begin{aligned}
&v_i =x+(h+\hat{h})\sum_{j=1}^I\frac{a_{ij}h}{h+\hat{h}}g(v_j)\quad i=1,\cdots ,I,\\
& v_i = x+(h+\hat{h})\sum_{j=1}^I\frac{b_{j}h}{h+\hat{h}}g(v_j) + (h+\hat{h})\sum_{j=I+1}^{I+\hat{I}}\frac{\hat{a}_{ij}\hat{h}}{h+\hat{h}}g(v_j)\quad i=I+1,\cdots ,I+\hat{I},\\
&\hat{\Phi}_{\hat{h}}\circ \Phi_h(x) = x+(h+\hat{h})\sum_{j=1}^I\frac{b_{j}h}{h+\hat{h}}g(v_j) + (h+\hat{h})\sum_{j=I+1}^{I+\hat{I}}\frac{\hat{b}_{j}\hat{h}}{h+\hat{h}}g(v_j).
\end{aligned}
\end{equation*}
According to this fact, we obtain that, under the notations and conditions of \cref{lem:rk},
$(\Phi_{h,g})^S$, $(\Phi_{h,\hat{g}})^S$ are analytic for $S|h|\leq h_0'=r/(4\frac{\kappa+(S-1)\mu}{S} m)$ and
\begin{equation*}
\norm{(\Phi_{h,\hat{g}})^S-(\Phi_{h,g})^S}_{\mathcal{K}}\leq 2S\mu |h|\norm{\hat{g}-g}_{\mathcal{B}(\mathcal{K},|h|(\kappa+(S-1)\mu) m)}.
\end{equation*}
Furthermore, for $S|h|< h_1 \leq h_0'$,
\begin{equation*}
\norm{\hat{g}-g}_{\mathcal{K}} \leq \frac{\norm{(\Phi_{h,\hat{g}})^S-(\Phi_{h,g})^S}_{\mathcal{K}}}{S|h|} +\frac{2S\mu|h|\norm{\hat{g}-g}_{\mathcal{B}(\mathcal{K},h_1\frac{\kappa+(S-1)\mu}{S} m)}}{h_1-S|h|} .
\end{equation*}
Tanking $h_0:=h_0'/S$ implies that Runge-Kutta methods satisfy \cref{asm:int} with $b_1=\frac{1}{4(\kappa+(S-1)\mu)}$, $b_2 = 2S\mu$, $b_3 = \kappa+(S-1)\mu$.

\subsection{Estimation of the Truncation}
The series in (\ref{eq:imde}) does not converge in general and needs to be truncated. Inspired by the induction idea for conventional modified equations in \cite{reich1999backward}, we prove the truncation estimation for the IMDE scenario below.
\begin{lemma}\label{lem:difference}
Let $f(y)$ be analytic in $\mathcal{B}(\mathcal{K}, r)$ and satisfies $\norm{f}_{\mathcal{B}(\mathcal{K}, r)} \leq m$. Suppose the $p$th-order numerical integrator $\Phi_{h}$ satisfies condition (\ref{con:intbound}). Take $\eta = \max\{6, \frac{b_2+1}{29}+1\}$, $\zeta = 10(\eta-1)$, $q = -\ln(2 b_2)/ \ln 0.912$ and let $K$ be the largest integer satisfying
\begin{equation*}
\frac{\zeta(K-p+2)^q|h|m}{\eta r} \leq e^{-q}.
\end{equation*}
If $|h|$ is small enough such that $K\geq p$, then the truncated IMDE satisfies
\begin{equation*}
\begin{aligned}
&\norm{\Phi_{h,f_h^{K}}-\phi_{h,f}}_{\mathcal{K}} \leq b_2\eta m e^{2q-qp}|h|e^{-\gamma /|h|^{1/q}},\\
&\norm{f_{h}^K - f}_{\mathcal{K}}\leq b_2 \eta m \brac{\frac{\zeta m}{b_1r }}^p (1+ 1.38^q d_p) |h|^p,\\
& \norm{f_{h}^K}_{\mathcal{K}} \leq (\eta-1)m,
\end{aligned}
\end{equation*}
where $\gamma = \frac{q}{e}\brac{\frac{b_1r}{\zeta m}}^{1/q}$, $d_p = p^{qp} e^{-q(p-1)}$.
\end{lemma}

\begin{proof}
For $0 \leq \alpha< 1$ and $|h| \leq h_0 = b_1(1-\alpha)r/m$, the condition (\ref{con:intbound}), together with the fact that $\mathcal{B}(\mathcal{B}(\mathcal{K}, \alpha r), (1-\alpha) r) = \mathcal{B}(\mathcal{K}, r)$ implies
\begin{equation*}
\begin{aligned}
\norm{\Phi_{h,f}-\phi_{h,f}}_{\mathcal{B}(\mathcal{K}, \alpha r)}
\leq& \norm{\Phi_{h,f}-I_D}_{\mathcal{B}(\mathcal{K}, \alpha r)}+ \norm{\phi_{h,f}-I_D}_{\mathcal{B}(\mathcal{K}, \alpha r)}\\
\leq& (b_2+1)|h|m \leq b_1(b_2+1)(1-\alpha)r.
\end{aligned}
\end{equation*}
Here, the map $\Phi_{h,f}-\phi_{h,f}$ contains the factor $h^{p+1}$ since $\Phi_{h,f}$ is of order $p$. By the maximum principle for analytic functions, we obtain that
\begin{equation*}
\norm{\frac{\Phi_{h,f}-\phi_{h,f}}{h^{p+1}}}_{\mathcal{B}(\mathcal{K}, \alpha r)} \leq \frac{b_1(b_2+1)(1-\alpha)r}{h_0^{p+1}}.
\end{equation*}
The calculation procedure of $f_k$ can be rewritten as the following recursion:
\begin{equation}\label{recursion}
f_k = \lim_{h\rightarrow 0} \frac{\phi_{h,f} - \Phi_{h,f_h^{k-1}}}{h^{k+1}}.
\end{equation}
Therefore, we deduce that
\begin{equation}\label{fp}
\begin{aligned}
\norm{f_p}_{\mathcal{B}(\mathcal{K}, \alpha r)} \leq& \frac{b_1(b_2+1)(1-\alpha)r}{h_0^{p+1}}
= (b_2+1)m\brac{\frac{m}{b_1(1-\alpha)r}}^{p}.
\end{aligned}
\end{equation}
Below we proceed to prove that for $\alpha \in [0,1)$, if
\begin{equation*}
|h|\leq h_k:=\frac{b_1(1-\alpha)r}{\zeta(k-p+1)^qm},
\end{equation*}
then
\begin{equation}\label{fk}
\norm{f_k}_{\mathcal{B}(\mathcal{K}, \alpha r)} \leq b_2\eta m\brac{\frac{\zeta(k-p+1)^qm}{b_1(1-\alpha)r}}^{k}
\end{equation}
for $k \geq p$ by induction, where $\eta = \max\{6, \frac{b_2+1}{29}+1\}$, $\zeta = 10(\eta-1)$ and $q = -\ln(2 b_2)/ \ln 0.912$. First, the case when $k=p$ is obvious since inequality (\ref{fp}). Suppose now (\ref{fk}) holds for $k \leq K$. If $|h| \leq h_{K+1}$, taking
\begin{equation*}
\delta_{K+1} := \frac{\eta-1}{(K-p+2)^q\zeta} ,\ \beta_K := (1-\delta_{K+1})(K-p+2)^q
\end{equation*}
yields that for $p \leq k \leq K$,
\begin{equation*}
|h|\leq \frac{b_1(1-\alpha)r}{\zeta(K-p+2)^qm} \leq \frac{b_1(1-\alpha-\delta_{K+1}(1-\alpha))r}{\zeta(k-p+1)^qm}.
\end{equation*}
Therefore, by inductive hypothesis and replacing $\alpha$ by $\alpha+\delta_{K+1}(1-\alpha) \in [\delta_{K+1}, 1)$ in (\ref{fk}), we obtain
\begin{equation*}
\norm{f_{k}}_{\mathcal{B}(\mathcal{K}, (\alpha+\delta_{K+1} (1-\alpha)) r)} \leq b_2\eta m\brac{\frac{\zeta(k-p+1)^qm}{b_1(1-\delta_{K+1})(1-\alpha)r}}^{k}.
\end{equation*}
This indicates that
\begin{equation*}
\begin{aligned}
&\norm{f_{h}^K}_{\mathcal{B}(\mathcal{K}, (\alpha+\delta_{K+1} (1-\alpha)) r)} \leq m\left[1+ (b_2+1) \brac{\frac{1}{\zeta\beta_K}}^p + b_2 \eta \sum_{k=p+1}^K\brac{\frac{(k-p+1)^q}{\beta_K}}^{k} \right].\\
\end{aligned}
\end{equation*}
Since
\begin{equation*}
\sum_{k=p+1}^K\brac{\frac{k-p+1}{K-p+1.9}}^{k} \leq 0.912,
\end{equation*}
which is maximal for $K=6$ and $p=1$, and
\begin{equation*}
\sum_{k=p+1}^K\brac{\frac{(k-p+1)^q}{\beta_K}}^{k} \leq \left[\sum_{k=p+1}^K\brac{\frac{k-p+1}{K-p+1.9}}^{k}\right]^q,
\end{equation*}
we deduce that
\begin{equation}\label{fKbound}
\norm{f_{h}^K}_{\mathcal{B}(\mathcal{K}, (\alpha+\delta_{K+1} (1-\alpha)) r)} \leq (\eta-1)m.
\end{equation}
Here, we have used the definition of $\eta$, $\zeta$ and $q$.
Subsequently, by this estimate and condition (\ref{con:intbound}), we obtain
\begin{equation*}
\begin{aligned}
\norm{\Phi_{h,f_h^K}-I_D}_{\mathcal{B}(\mathcal{K}, \alpha r)} \leq& |h| b_2\norm{f_{h}^K}_{\mathcal{B}(\mathcal{K}, (\alpha+\delta_{K+1} (1-\alpha)) r)}\leq |h| b_2 (\eta-1)m,
\end{aligned}
\end{equation*}
where
\begin{equation*}
|h|\leq h_K = \frac{b_1(1-\alpha)r}{\zeta(K-p+2)^qm} = \frac{b_1\delta_{K+1}(1-\alpha)r}{(\eta-1)m}.
\end{equation*}
And then using the triangle inequality yields that
\begin{equation*}
\begin{aligned}
\norm{\Phi_{h,f_h^K}-\phi_{h,f}}_{\mathcal{B}(\mathcal{K}, \alpha r)}& \leq \norm{\Phi_{h,f_h^K}-I_D}_{\mathcal{B}(\mathcal{K}, \alpha r)} + \norm{\phi_{h,f}-I_D}_{\mathcal{B}(\mathcal{K}, \alpha r)} \leq h_K b_2 \eta m.\\
\end{aligned}
\end{equation*}
Again by the maximum principle for analytic functions, together with the fact that $\Phi_{h,f_h^K}-\phi_{h,f}$ contains the factor $h^{K+2}$, we deduce that
\begin{equation}\label{phi}
\begin{aligned}
\norm{\frac{\Phi_{h,f_h^K}-\phi_{h,f}}{h^{K+2}}}_{\mathcal{B}(\mathcal{K}, \alpha r)}& \leq \frac{h_K b_2 \eta m}{h_K^{K+2}}.\\
\end{aligned}
\end{equation}
Again by (\ref{recursion}), we conclude that
\begin{equation*}
\norm{f_{K+1}}_{\mathcal{B}(\mathcal{K}, \alpha r)} \leq b_2 \eta m\brac{\frac{\zeta (K-p+2)^qm}{b_1(1-\alpha)r}}^{K+1},
\end{equation*}
which completes the induction.

The above induction also shows that (\ref{phi}) holds if $|h| \leq h_{K+1}$. Taking $\alpha = 0$ we have
\begin{equation*}
\norm{\Phi_{h,f_h^K}-\phi_{h,f}}_{\mathcal{K}} \leq b_2\eta |h|m\brac{ \frac{\zeta(K-p+2)^q|h|m}{b_1r}}^{K+1}.
\end{equation*}
We set $K^*$ to be the largest integer satisfying
\begin{equation*}
\frac{\zeta (K^*-p+2)^q|h|m}{b_1r} \leq e^{-q}.
\end{equation*}
Clearly, $|h| \leq h_{K^*+1}$ with $\alpha=0$. Therefore,
\begin{equation*}
\begin{aligned}
\norm{\Phi_{h,f_h^{K^*}}-\phi_{h,f}}_{\mathcal{K}} \leq& b_2\eta |h|m\brac{ \frac{\zeta(K^*-p+2)^q|h|m}{b_1r}}^{K^*+1}\\
\leq& b_2\eta m e^{2q-qp}|h|e^{-\gamma /|h|^{1/q}},
\end{aligned}
\end{equation*}
where $\gamma = \frac{q}{e}\brac{\frac{b_1r}{\zeta m}}^{1/q}$. The first part of the lemma has been completed.

Next, according to (\ref{fk}) we obtain
\begin{equation*}
\begin{aligned}
\norm{f_{h}^{K^*} - f}_{\mathcal{K}}
\leq& b_2 \eta m \brac{\frac{\zeta |h|m}{b_1r }}^p\Big[1+\sum_{k=p+1}^{K^*}\frac{(k-p+1)^{qp}}{e^{q(k-p)}}\brac{\frac{k-p+1}{{K^*}-p+2}}^{q(k-p)}\Big]\\
\leq& b_2 \eta m \brac{\frac{\zeta m}{b_1r }}^p (1+ 1.38^q d_p) |h|^p,\\
\end{aligned}
\end{equation*}
where $d_p = p^{qp} e^{-q(p-1)}$ satisfies $d_p \geq (k-p+1)^{qp}e^{-q(k-p)}$ for any $k\geq p+1$.

Finally, we immediately derive the bound of $f_h^K$ due to (\ref{fKbound}). The proof has been completed.
\end{proof}

\subsection{Proof of \cref{thm:imde}}

\begin{proof}
Since the ODE solver satisfy \cref{asm:int}, regarding the ODE solver as a one-step integrator and applying the first inequality of \cref{lem:difference}, we have that if $T \leq T_0 :=\eta r_2/((2e)^q\zeta m)$,
\begin{equation*}
\norm{\brac{\Phi_{h,F_{Sh}^{K}}}^S-\phi_{Sh,f}}_{\mathcal{B}(x,r_1)} \leq b_2\eta m e^{q}She^{-\gamma' /(Sh)^{1/q}},
\end{equation*}
where $h>0$, $b_2$ is the coefficient defined in \cref{asm:int} and $\eta, q, \gamma$ are given by \cref{lem:difference}

By \cref{lem:inmde}, $F_{Sh}^{K}=f_h^K$. And thus we obtain that
\begin{equation}\label{Phi fnet-fh}
\delta : = \frac{1}{Sh}\norm{\brac{\Phi_{h, f_{\theta}}}^S-\brac{\Phi_{h,f_h^K}}^S}_{\mathcal{B}(x, r_1)} \leq \mathcal{L} + c m e^{-\gamma /h^{1/q}},
\end{equation}
where $\gamma =\gamma'/S^{1/q}= \frac{q}{e}\brac{\frac{b_1r_2}{S\zeta m}}^{1/q}$, $c = b_2\eta e^q$.
Next, by the third inequality of \cref{lem:difference}, $\norm{f_{h}^K}_{\mathcal{B}(x, r_1)}<(\eta-1)m$. Let
\begin{equation*}
h_1 = (eb_2+1)h, \ \lambda = \frac{b_2h}{h_1-h} = e^{-1}, \ M = (\eta-1)m.
\end{equation*}
Using the third item of \cref{asm:int}, we deduce that for $0\leq j \leq r_1/h_1b_3 M$,
\begin{equation*}
\begin{aligned}
\norm{f_{\theta} - f_h^K}_{\mathcal{B}(x, jh_1b_3 M)} \leq \delta + \lambda \norm{f_{\theta} - f_h^K}_{\mathcal{B}(x, (j+1)h_1b_3 M)}.
\end{aligned}
\end{equation*}
This yields
\begin{equation*}
\begin{aligned}
\norm{f_{\theta} - f_h^K}_{\mathcal{B}(x, jh_1b_3 M)} - \frac{\delta}{1-\lambda}
\leq \lambda\brac{ \norm{f_{\theta} - f_h^K}_{\mathcal{B}(x, (j+1)h_1b_3 M)} - \frac{\delta}{1-\lambda}}.
\end{aligned}
\end{equation*}
Using this estimate iteratively, we deduce that
\begin{equation*}\label{fnet-fh}
\begin{aligned}
\norm{f_{\theta}(x) - f_h^K(x)} \leq e^{-\hat{\gamma}/h} \norm{f_{\theta} - f_h^K}_{\mathcal{B}(x, r_1)} + \frac{\delta}{1-\lambda},
\end{aligned}
\end{equation*}
where $\hat{\gamma} = \frac{r_1}{(eb_2+1)b_3M}$. By this estimation and (\ref{Phi fnet-fh}), we conclude that
\begin{equation*}
\norm{f_{\theta}(x) - f_h^K(x)} \leq c_1 m e^{-\gamma /h^{1/q}} + C \mathcal{L},
\end{equation*}
where $C =e/(e-1)$ and $c_1$ is a constant satisfying $c_1\geq C \cdot c+ \eta e^{\gamma/h^{1/q} - \hat{\gamma}/h}$.
\end{proof}

\subsection{Proof of \cref{thm:imdeerror}}\label{sec:imdeerror}
\begin{proof}
The first part has been proved in \cref{lem:modiode}, and the second part is a direct consequence of \cref{thm:imde}, the second inequality of \cref{lem:difference} and Triangle Inequality.
\end{proof}

\subsection{Proof of \cref{cor:trajerror}.}\label{sec:trajerror}
We first state a version of the well-known Gr{\"o}nwall's Inequality \cite{howard1998gronwall}.
\begin{proposition}\label{pro:gronwall}
Let $U\subset \R^D$ be an open set, Let $g_1,g_2 :U\rightarrow \R^D$ be continuous functions and let $y, \tilde{y}:[t_0, t_1]\rightarrow U$ satisfy
\begin{equation*}
\frac{d}{dt}y(t) = f(y(t)),\ y(0)=x,\quad\text{and}\quad \frac{d}{dt}\tilde{y}(t) = \tilde{f}(\tilde{y}(t)),\  \tilde{y}(0)=\tilde{x}.
\end{equation*}
Assume $\tilde{f}$ is Lipschitz with Lipschitz constant $C_1$ and $\norm{f(y(t))-\tilde{f}(y(t))} \leq \varphi(t)$ for continuous function $\varphi$. Then, for $t\in [t_0, t_1]$, 
\begin{equation*}
\norm{y(t) - \tilde{y}(t)}\leq e^{C_1(t-t_0)}\norm{x-\tilde{x}} +  e^{C_1(t-t_0)}\int_{t_0}^t e^{-C_1(\tau-t_0)}\varphi(\tau) d\tau.
\end{equation*}
\end{proposition}
\begin{proof}[Proof of \cref{cor:trajerror}]
Consider the following two equations
\begin{equation*}
\frac{d}{dt}y(t) = f(y(t)),\ y(0)=x,\quad\text{and}\quad \frac{d}{dt}\tilde{y}(t) = f_{\theta}(\tilde{y}(t)),\  \tilde{y}(0)=x.
\end{equation*}
We denote the set of the points on exact trajectory as $V_t = \{\phi_{\tau, f}(x)|0\leq \tau \leq t\}$. By \cref{thm:imdeerror}, there exist constant $C_2$ such that
\begin{equation*}
\norm{f(y(t)) - f_{\theta}(y(t))} \leq C_2(h^p+ \mathcal{L}_t),\quad \text{where}\ \mathcal{L}_t = \norm{\brac{\Phi_{h,f_{\theta}}}^S -\phi_{Sh,f}}_{\mathcal{B}(V_t, r_1)}/(Sh).
\end{equation*}
Therefore, by \cref{pro:gronwall},
\begin{equation*}
\norm{\phi_{t,f} (x) - \phi_{t,f_{\theta}}(x)} \leq e^{C_1t}\int_{0}^t e^{-C_1\tau}C_2(h^p+ \mathcal{L}_{\tau}) d\tau \leq C_2(h^p+ \mathcal{L}_{t})\cdot  e^{C_1t}\int_{0}^t e^{-C_1\tau} d\tau \leq  \frac{C_2e^{C_1t}-C_2}{C_1}(h^p+ \mathcal{L}_{t}),
\end{equation*}
which concludes the proof.
\end{proof}

\subsection{Proof of \cref{lem:imdeh}}\label{sec:imdeh}
For even dimension $D$, denote the $D/2$-by-$D/2$ identity matrix by $I$, and let
\begin{equation*}
J=\begin{pmatrix} 0 & I \\ -I & 0\end{pmatrix}.
\end{equation*}
\begin{definition}
A differentiable map $g : U \rightarrow \R^{D}$ (where $D$ is even and $U\subseteq \R^{D}$ is an open set) is called symplectic if
\begin{equation*}
g'(x)^{T}Jg'(x)=J,
\end{equation*}
where $g'(x)$ is the Jacobian of $g(x)$.
\end{definition}
A Hamiltonian system is given by
\begin{equation}\label{eq:Hami}
\frac{d}{dt}y = J^{-1} \nabla H(y), \quad y(0)=x,
\end{equation}
where $y \in \R^D$ and $H$ is the Hamiltonian function typically representing the energy of (\ref{eq:Hami}) \cite{arnold2013mathematical, arnold2007mathematical}. A remarkable property of Hamiltonian system is the symplecticity of the phase flow, which was proved by Poincar\'{e} in 1899 \cite{arnold2013mathematical}, i.e.,
\begin{equation*}
\phi_t'(x)^T J \phi_t'(x) = J,
\end{equation*}
where $\phi_t'(x) = \frac{\partial \phi_t(x)}{\partial x}$ is the Jacobian of $\phi_t$. Due to the intrinsic symplecticity, it is natural to search for numerical methods that preserve symplecticity, i.e., make $\Phi_h$ be a symplectic map. Such numerical methods are called symplectic methods, see e.g., \cite{feng1984on, feng1986difference, hairer2006geometric}.

\begin{proof}[Proof of \cref{lem:imdeh}]
For a Hamiltonian system (\ref{eq:Hami}), the target function $f$ obeys $f(y)=J^{-1}\nabla H(y)$, which yields $f_0 =J^{-1}\nabla H(y)$. If the employed numerical integrator $\Phi_h$ is symplectic, suppose $f_k(y)=J^{-1}\nabla H_k(y)$ for $k=1,2,\cdots,K$, we need to prove the existence of $H_{K+1}(y)$ satisfying
\begin{equation*}
f_{K+1}(y)=J^{-1}\nabla H_{K+1}(y).
\end{equation*}
By induction, the truncated IMDE
\begin{equation*}
\frac{d}{dt}\tilde{y}=f_{h}^K(\tilde{y})=f(\tilde{y})+hf_1(\tilde{y})+h^2f_2(\tilde{y})+ \cdots +h^{K}f_K(\tilde{y})
\end{equation*}
has the Hamiltonian $H(\tilde{y})+hH_1(\tilde{y})+\cdots+h^{K}H_K(\tilde{y})$. For arbitrary initial value $x$, the numerical solution $\Phi_{h,f_{h}^K}(x)$ satisfies
\begin{equation*}
\phi_{h,f}(x)= \Phi_{h,f_{h}^K}(x)+h^{K+2}f_{K+1}(x)+\mathcal{O}(h^{K+3}).
\end{equation*}
And thus
\begin{equation*}
\phi_{h,f}'(x)= \Phi_{h,f_{h}^K}'(x)+h^{K+2}f_{K+1}'(x)+\mathcal{O}(h^{K+3}).
\end{equation*}
According to the facts that $\phi_{h,f}$ and $\Phi_{h,f_{h}^K}$ are symplectic maps, and $\Phi_{h,f_{h}^K}'(x)=I+\mathcal{O}(h)$, we have
\begin{equation*}
\begin{aligned}
J=&\phi_{h,f}'(x)^TJ\phi_{h,f}'(x)=J+h^{K+2}(f_{K+1}'(x)^TJ+Jf_{K+1}'(x))+\mathcal{O}(h^{K+3}).
\end{aligned}
\end{equation*}
Consequently, $f_{K+1}'(x)^TJ+Jf_{K+1}'(x)=0$, i.e., $Jf_{K+1}'(x)$ is symmetric. According to the Poincar\'{e} Lemma (see e.g., Lemma \uppercase\expandafter{\romannumeral 6}.2.7 of \cite{hairer2006geometric}), for any $x$, there exists a neighbourhood and a smooth function $H_{K+1}$ obeying
\begin{equation*}
Jf_{K+1}=\nabla H_{K+1}
\end{equation*}
on this neighbourhood. Hence the induction holds and the first part of the proof is completed.

If the employed numerical integrator $\Phi_h$ is not symplectic, we suppose $\Phi_h$ preserves symplectic form of order $\hat{k}$, i.e.,
\begin{equation*}
\Phi_{h}'(x)^TJ\Phi_{h}'(x) =J+ \mathcal{O}(h^{\hat{k}+1})
\end{equation*}
when the method is applied to Hamiltonian systems. By repeating the above induction, we can prove that $f_k(y)=J^{-1}\nabla H_k(y)$ for $k=1,2,\cdots,\hat{k}-1$. Subsequently, since
$\phi_{h,f}'(x)= \Phi_{h,f_{h}^{\hat{k}-1}}'(x)+h^{\hat{k}+1}f_{\hat{k}}'(x)+\mathcal{O}(h^{\hat{k}+2})$ and $\phi_{h,f}'(x)^TJ\phi_{h,f}'(x)=J$,
we have that
\begin{equation*}
J=\phi_{h,f}'(x)^TJ\phi_{h,f}'(x)=J+\mathcal{O}(h^{\hat{k}+1}) +h^{\hat{k}+1}(f_{\hat{k}}'(x)^TJ+Jf_{\hat{k}}'(x))+\mathcal{O}(h^{\hat{k}+2}).
\end{equation*}
Consequently, $f_{\hat{k}}'(x)^TJ+Jf_{\hat{k}}'(x)\neq0$, i.e., $Jf_{\hat{k}}'(x)$ is not symmetric. This fact yields that $Jf_{\hat{k}}$ is not a potential field since the Jacobian of a potential field must be symmetric. The proof is completed.
\end{proof}

\section{Experimental Details}\label{sec:Experimental Details}

Since both true $f$ and the IMDE $f_h$ are inaccessible in practice, we consider several benchmark problems that are widely investigated for the discovery of hidden dynamics \cite{du2021discovery, greydanus2019hamiltonian, keller2021discovery, yu2021onsagernet}. Here, the true system is known and we can calculate the corresponding IMDE. We use solvers with different levels of accuracy to train Neural ODE and we use a Runge-Kutta method of order $4$, denoted as RK4, as the test solver. The code accompanying this paper are publicly available at \url{https://github.com/Aiqing-Zhu/IMDE}.

The benchmark problems are the pendulum system, the damped harmonic oscillator and the nonlinear Lorenz system, which are respectively formulated as
\begin{equation*}
\left\{\begin{aligned}
&\frac{d}{dt}y_1=-10\sin y_2, \\
&\frac{d}{dt}y_2=y_1,
\end{aligned}\right.
\qquad
\left\{\begin{aligned}
&\frac{d}{dt}y_1 = -0.1y_1^3+2.0y_2^3,\\
&\frac{d}{dt}y_2 = -2.0y_1^3-0.1y_2^3,
\end{aligned}\right.
\qquad
\left\{\begin{aligned}
\frac{d}{dt}y_1 =&10(y_2-y_1),\\
\frac{d}{dt}y_2 =&y_1(28-10y_3)-y_2,\\
\frac{d}{dt}y_3 =&10y_1y_2-\frac{8}{3}y_3.
\end{aligned}\right.
\end{equation*}
The training dataset consists of grouped pairs of points with shared data step $T$, i.e., $\mathcal{T} = \{(x_n, \phi_{T}(x_n))\}_{n=1}^N$. On all experiments, the neural networks employed in Neural ODE are all fully connected networks with two hidden layers, each layer having 128 hidden units. The activation function is chosen to be tanh. We optimize the mean-squared-error loss
\begin{equation*}
\frac{1}{N}\sum_{n=1}^{N} \|(\Phi_{\frac{T}{S}, f_{\theta}})^S (x_n) -\phi_{T}(x_{n})\|^2
\end{equation*}
for $3 \times 10^{5}$ epochs with Adam optimization \cite{kingma2014adam} where the learning rate is set to decay exponentially with linearly decreasing powers from $10^{-2}$ to $10^{-5}$.

For the first two benchmarks, we take $N=10000$ and randomly sample $x_n$ from $[-3.8, 3.8]\times[-1.2, 1.2]$ and $[-2.2, 2.2]\times[2.2,2.2]$, where $T$ is chosen to be $0.04$ and $0.02$, respectively. For the Lorenz system, the training dataset consists of $N=251$ data points on a single trajectory starting from $(-0.8,0.7,2.6)$ with shared data step of $T=0.04$, i.e., $x_1,\cdots, x_{N+1}$ where $x_n = \phi_{nT}(x_0)$. These data points are grouped into pairs before training, and denoted as $\mathcal{T} = \{(x_n, x_{n+1})\}_{n=1}^N$. After training, we plot the trajectories of the benchmark problems starting at $(0,1)$, $(2,0)$, $(-0.8,0.7,2.6)$, respectively. For comparison, the first components of the trajectories are presented in \cref{fig:tras}.

To investigate errors versus $h$ for the first benchmark problem, we take $N=10000$ and randomly sample $x_n$ from $[-3.8, 3.8]\times[-1.2, 1.2]$. We take multiple $T$ to generate corresponding training data, and use the Euler method and the explicit midpoint rules to train the model on these data, where the composition numbers are set to be $S=1,\cdots,6$. After training, we calculate the mean absolute error between $f_{\theta}$ and $f$ via
\begin{equation*}
\frac{1}{2000}\sum_{x} \norm{f_{\theta}(x) - f(x)}_{\infty},
\end{equation*}
where $x$ is randomly sampled from $[-3.8, 3.8]\times[-1.2, 1.2]$. The mean error based on $5$ independent experiments are recorded in \cref{fig:error}. We calculate the order of $Error$ with respect to discrete step $h$ by $\log_2(\frac{Error(2h)}{Error(h)})$.

The first benchmark problem is a Hamiltonian system. To investigate the behavior of learning Hamiltonian system using Neural ODE, we test the model trained for investigating errors. Here, we take $T=0.12$. As for the ODE solver, we take the Euler method with $S=6$ and the explicit midpoint rule with $S=1$. We select the trained models with the above parameters and depict the orbits starting from $(0,1)$ in \cref{fig:hami}.

\end{document}